%% file: paper.tex
\documentclass[conference]{IEEEtran}

\usepackage{cite} 
\usepackage{amsthm,amsmath, amsfonts, amssymb} 
\newtheorem{theorem}{Theorem}[section]
\newtheorem{proposition}{Proposition}[section]
\usepackage[utf8]{inputenc}
\usepackage{enumitem}
\usepackage{subfig}
\usepackage{hyperref}
\usepackage{xcolor}
\usepackage{bm,graphicx, enumerate}
\usepackage{verbatim}
\usepackage{algorithm, algorithmic}
\usepackage{dsfont}

\input{macro+variables}

\usepackage{bm}

\begin{document}
%
\title{Graph sampling with determinantal processes}

\author{\IEEEauthorblockN{Nicolas Tremblay\IEEEauthorrefmark{1},
Pierre-Olivier Amblard\IEEEauthorrefmark{1} and 
Simon Barthelm\'e\IEEEauthorrefmark{1}}
\IEEEauthorblockA{\IEEEauthorrefmark{1}CNRS, GIPSA-lab, Grenoble-Alpes University, Grenoble, France}}

\maketitle

\begin{abstract}
We present a new random sampling strategy for $k$-bandlimited signals defined on graphs, based on determinantal point processes (DPP).  For small graphs, \ie, in cases where the spectrum of the graph is accessible, we exhibit a DPP sampling scheme that enables perfect recovery of bandlimited signals. For large graphs, \ie, in cases where the graph's spectrum is not accessible, we investigate, both theoretically and empirically, a sub-optimal but much faster DPP based on loop-erased random walks on the graph. Preliminary experiments show promising results especially in cases where the number of measurements should stay as small as possible and for graphs that have a strong community structure. Our sampling scheme is efficient and can be applied to graphs with up to $10^6$ nodes. 
\end{abstract}



\section{Introduction}
Graphs are a central modelling tool for network-structured data. 
Data on a graph, called graph signals~\cite{shuman_emerging_2013}, such as individual hobbies in social networks, blood flow of brain regions in neuronal networks, or traffic at a station in transportation networks, 
may be represented by a scalar per node. Studying such signals with respect to the particular graph topology on which it is defined, is the 
goal of graph signal processing (GSP). One of the top challenges of GSP is not only to adapt classical processing tools to arbitrary topologies defined  
by the underlying graph, but to do so with efficient algorithms that can handle the large datasets encountered today. 
A possible answer to this efficiency challenge is dimension reduction, and in particular sampling. 
Graph sampling~\cite{pesenson_sampling_2008, anis_efficient_2016, chen_discrete_2015, puy_random_2016, gama_rethinking_2016-1} consists in measuring a graph signal on a reduced set of nodes carefully chosen to optimize a predefined objective, may it be the estimation of some statistics, signal compression via graph filterbanks, signal recovery, etc.

In an effort to generalize Shannon's sampling theorem of bandlimited signals, several authors~\cite{anis_efficient_2016, chen_discrete_2015, puy_random_2016} have proposed sampling schemes adapted to the recovery of \textit{a priori} smooth graph signals, and in particular 
 $k$-bandlimited signals, \ie, linear combinations of only $k$ graph Fourier modes. Classically, there are three main types of sampling schemes that one may try to adapt to graphs: periodic, irregular, and random. 
 
 Periodic graph sampling is ill-defined unless the graph is not bipartite (or multipartite) and has only been investigated in the design of graph filterbanks where the objective is to sample ``one every two nodes''~\cite{shuman_multiscale_2016, narang_compact_2013}. 

Irregular graph sampling of $k$-bandlimited graph signals has been studied by Anis \emph{et al.}~\cite{anis_efficient_2016} and Chen \emph{et al.}~\cite{chen_discrete_2015} who, building upon  Pensenson's work~\cite{pesenson_sampling_2008}, look for optimal sampling sets of $k$-bandlimited graphs, in the sense that they are tight (\ie, they contain only $k$ nodes) and that they maximize noise-robustness of the reconstruction. To exhibit such an optimal set, one needs i) to compute the first $k$ eigenvectors of the Laplacian which becomes prohibitive at large scale, and ii) to optimize over all subsets of $k$ nodes, which is a large combinatorial problem. Using greedy algorithms and spectral proxies~\cite{anis_efficient_2016} to sidestep some of the computational bottlenecks, these methods enable to efficiently perform optimal sampling on graphs of size up to a few thousands nodes.


Loosening the tightness constraint on sampling sets, random graph sampling~\cite{puy_random_2016} shows that if $O(k\log{k})$ nodes are drawn with replacement according to a particular probability distribution $p^*$ that depends on the first $k$ eigenvectors of the Laplacian, then recovery is guaranteed with high probability. $p^*$  was shown to be closely related to \textit{leverage scores}, an important concept in randomized numerical linear algebra~\cite{drineas_randnla:_2016}. Moreover, an efficient estimation of $p^*$, bypassing expensive spectral computations, enables to sample very large graphs; as was shown in~\cite{tremblay_compressive_2016}, where experiments were performed on graphs with $N\simeq 10^6$ and $k\simeq 200$. 

In this paper, we detail another form of random graph sampling,  motivated by a simple experimental context. Consider a graph composed of two almost disconnected communities of equal size. Its first Fourier mode is constant (as always), and its second one is positive in one community, and negative in the other, while its absolute value is more or less constant on the whole graph. Therefore, a 2-bandlimited graph signal will approximately be equal to a constant in one community, and to another constant in the other: any sampling set composed of one node drawn from each community will enable to recover the whole signal. Unfortunately, random graph sampling as proposed in~\cite{puy_random_2016} samples independently with replacement: it has a fair chance of sampling only in one community, thereby missing half of the information. To avoid such situations, we propose to take advantage of determinantal point processes (DPP), random processes that introduce negative correlations to force diversity within the sampling set. 

\textbf{Contributions.} If the first $k$ eigenvectors of the Laplacian are computable, we exhibit a DPP that always samples sets of size $k$ perfectly embedding $k$-bandlimited graph signals. Moreover, in the case the graph is too large to compute its first $k$ eigenvectors, and building upon recent results linking determinantal processes and random rooted spanning forests~\cite{avena_random_2013}, we  show to what extent a simple algorithm based on loop-erased random walks  is an acceptable approximation. 

\section{Background}
\label{sec:background}
Consider an undirected weighted graph $\mathcal{G}$ of $\nbVert$ nodes represented by its adjacency matrix $\ma{W}$, where $\ma{W}_{ij}=\ma{W}_{ji}\geq0$ is the weight of the connection between nodes $i$ and $j$. Denote by $\Lap = \ma{D} - \ma{W}$ its Laplacian matrix, where $\ma{D}$ is the diagonal matrix with entries $\ma{D}_{ii}=\sum_{j} \ma{W}_{ij}$. $\Lap$ is real and symmetrical, therefore diagonalizable as $\Fou\ma{\Lambda}\Fou^\adjoint$ with $\Fou=(\vec{u}_1|\ldots|\vec{u}_N)\in\mathbb{R}^{N\times N}$ the graph Fourier basis and $\ma{\Lambda} = \text{diag}(\lambda_1,\ldots,\lambda_N)$ the diagonal matrix of the associated frequencies, that we sort: $0=\lambda_1\leq\lambda_2\leq\ldots\leq\lambda_N$. 
For any function $h:\mathbb{R}^+\rightarrow\mathbb{R}$, we write $h(\ma{\Lambda})=\text{diag}(h(\lambda_1),\ldots,h(\lambda_N))$. The graph filter associated to the frequency response $h$ reads $\Fou h(\ma{\Lambda})\Fou^\adjoint$.  

Given the frequency interpretation of the graph Fourier modes, one may define ``smooth'' signals as linear combinations of the first few low-frequency Fourier modes. Writing $\Fou_k=(\vec{u}_1|\ldots|\vec{u}_k)\in\mathbb{R}^{N\times k}$, we have the formal definition:
\begin{definition}[$k$-bandlimited signal on $\mathcal{G}$] A signal $\sig \in \Rbb^{\nbVert}$ defined on the nodes of the graph $\Graph$ is $k$-bandlimited with $\nbClass \in \Nbb^*$ if $\sig \in \spann(\ma{U}_\nbClass)$, \ie, $\exists~\vec{\alpha}\in\mathbb{R}^k$ such that $\sig = \Fou_k\vec{\alpha}$.
\end{definition}
Sampling consists in selecting a subset $\mathcal{A}=(\omega_1,\ldots,\omega_m)$ of $m$ nodes of the graph and measuring the signal on it. To each possible sampling set, we associate a measurement matrix $\Meas=(\vec{\delta}_{\omega_1}|\vec{\delta}_{\omega_2}|\ldots|\vec{\delta}_{\omega_m})^\adjoint\in\mathbb{R}^{m\times N}$ where $\vec{\delta}_{\omega_i}(j)=1$ if $j=\omega_i$, and 0 otherwise. Given a $k$-bandlimited signal $\sig$, its (possibly noisy) measurement on $\mathcal{A}$  reads:
\begin{align}
\label{eq:system_rec}
 \meas = \Meas\sig + \vec{n} \in\mathbb{R}^m,
\end{align}
where $\vec{n}$ models any kind of measurement noise and/or the fact that $\sig$ may not exactly be in $\spann(\Fou_k)$. In this latter case, $\sig = \Fou_k\Fou_k^\adjoint\sig + \vec{e}$, where $\Fou_k\Fou_k^\adjoint\sig$ is the closest signal to $\sig$ that is in $\spann(\Fou_k)$; and $\Meas\sig = \Meas\Fou_k\Fou_k^\adjoint\sig+\vec{n}$ with  $\vec{n}=\Meas\vec{e}$.  

\subsection{If $\Fou_k$ is known}
\label{subsec:greedy_options}
If $\Fou_k$ is known, the measurement on $\mathcal{A}$ reads $\meas=\Meas\Fou_k\vec{\alpha} + \vec{n}$ and one recovers the signal by solving:
\begin{align}
\label{eq:recovery_Uk_known}
 \sig_{\text{rec}} = \argmin_{\vec{z}\in\spann(\Fou_k)} \norm{\Meas\vec{z}-\meas}^2
= \Fou_k  (\Meas\Fou_k)^\dag \meas,
\end{align}
where $(\Meas\Fou_k)^\dag\in\mathbb{R}^{k\times m}$ is the Moore-Penrose pseudo-inverse of $\Meas\Fou_k\in\mathbb{R}^{m\times k}$ and $\norm{.}$ the Euclidean norm. Of course, perfect  reconstruction is impossible if $m<k$, as the system would be underdetermined. Hereafter, we thus suppose $m\geq k$. Denote by $\sigma_1\leq\ldots\leq\sigma_k$ the singular values of $\Meas\Fou_k$.

\begin{proposition}[Classical result]\label{thm:sigma_0} If $\sigma_1>0$, then perfect reconstruction (up to the noise level) is possible.\end{proposition}
 \begin{proof}
If $\sigma_1>0$, then $\Fou_k^\adjoint\Meas^\adjoint\Meas\Fou_k$ is invertible; and since $\sig=\Fou_k\vec{\alpha}$, we have:
\begin{align*}
 \sig_{\text{rec}} &= \Fou_k  (\Meas\Fou_k)^\dag \meas = \Fou_k  (\Fou_k^\adjoint\Meas^\adjoint\Meas\Fou_k)^{-1} \Fou_k^\adjoint\Meas^\adjoint(\Meas \sig +\vec{n})\\
 &=\sig + \Fou_k (\Fou_k^\adjoint\Meas^\adjoint\Meas\Fou_k)^{-1} \Fou_k^\adjoint\Meas^\adjoint\vec{n}.\hspace{2.9cm}\qedhere
\end{align*}
\end{proof}

In the following, in the cases where $\Fou_k$ is known, we fix $m=k$. Chen et al.~\cite{chen_discrete_2015} showed that a sample of size $k$ always exists s.t. $\sigma_1>0$. In fact, in general, many possible such subsets exist. Finding the optimal one is a matter of definition. Authors in~\cite{chen_discrete_2015,anis_efficient_2016} propose two optimality definitions via noise robustness. A first option is to minimize the worst-case error, which translates to finding the subset that maximizes $\sigma_1^2$:
\begin{align}
\label{eq:optim_WCE}
\mathcal{A}^{\text{\footnotesize{WCE}}} = \arg\max_{\hspace{-0.7cm}\mathcal{A} \,\text{s.t.}\, |\mathcal{A}|=k}  \sigma_{1}^2.
\end{align}
A second option is to find the subset that minimizes the mean square error (assuming $\mathbb{E}(\vec{n}\vec{n}^\adjoint)$ is the identity) :
\begin{align}
\label{eq:optim_MSE}
 \hspace{-0.3cm}\mathcal{A}^{\text{\footnotesize{MSE}}} = \arg\min_{\hspace{-0.7cm}\mathcal{A} \,\text{s.t.}\, |\mathcal{A}|=k} \text{tr}\left[(\Fou_k^\adjoint\Meas^\adjoint\Meas\Fou_k)^{-1}\right]=\arg\min_{\hspace{-0.7cm}\mathcal{A} \,\text{s.t.}\, |\mathcal{A}|=k}\sum_{i=1}^k 
 \frac{1}{\sigma_i^2}, 
\end{align}
where $\text{tr}$ is the trace operator.

We argue that another possible definition of optimal subset is such that the determinant of $\Fou_k^\adjoint\Meas^\adjoint\Meas\Fou_k$ is maximal, \ie:
\begin{align}
\label{eq:optim_MV}
\displaystyle \mathcal{A}^{\text{\footnotesize{MV}}} = \arg\max_{\hspace{-0.7cm}\mathcal{A} \,\text{s.t.}\, |\mathcal{A}|=k}  \prod_{i=1}^k \sigma_i^2,
\end{align}
where MV stands for Maximum Volume, as the determinant is proportional to the volume spanned by the $k$ sampled lines of $\Fou_k$. This definition of optimality does not have a straight-forward noise robustness interpretation, even though it is related to the previous definitions (in particular, increasing $\sigma_1$ necessarily increases the determinant). Nevertheless, such volume sampling as it is sometimes called~\cite{deshpande_efficient_2010} has been proved optimal in the related problem of low-rank approximation~\cite{goreinov_maximal-volume_2001}, and is largely studied in diverse contexts~\cite{kulesza_determinantal_2012}.

In all three cases, finding the optimal subset is a very large combinatorial problem (it implies a search over all combinations of $k$ nodes):  in practice, greedy algorithms such as Alg.~\ref{alg:greedy} are used to find approximate solutions. In the following, we refer to $\hat{\mathcal{A}}^{\text{WCE}}$, $\hat{\mathcal{A}}^{\text{MSE}}$ and $\hat{\mathcal{A}}^{\text{MV}}$ the (deterministic) solutions obtained by greedy optimization of the above 3 objectives. We refer to $\hat{\mathcal{A}}^{\text{maxvol}}$ the solution to the $\texttt{maxvol}$ algorithm proposed in~\cite{goreinov_how_2008}, which is another approximation of $\mathcal{A}^{\text{MV}}$.

\begin{algorithm}[tb]
 \caption{Greedy algorithm to approximate~(\ref{eq:optim_WCE}),~(\ref{eq:optim_MSE}), or~(\ref{eq:optim_MV})}
 \label{alg:sampling}
\begin{algorithmic}
\label{alg:greedy}
\STATE \textbf{Input:} $\Fou_k$, \texttt{objective} = WCE, MSE or MV.\\
$\mathcal{Y}\leftarrow \emptyset$\\
\textbf{while} $|\mathcal{Y}|\neq k$ \textbf{do}:\\
\hspace{0.5cm}\textbf{for} $n=1,2,\ldots,N$ \textbf{do}:\\
\hspace{1cm} $\bm{\cdot}$ $\mathcal{Y}_n \leftarrow \mathcal{Y}\cup\{n\}$\\
\hspace{1cm} $\bm{\cdot}$ Compute $\{\sigma_i\}_{i=1,\ldots,|\mathcal{Y}_n|}$ the singular values of the\\\hspace{1.5cm} restriction of $\Fou_k$ to the lines indexed by $\mathcal{Y}_n$.\\
\hspace{1cm} $\bm{\cdot}$ Store the gain on the \texttt{objective} of adding $n$\\
\hspace{0.5cm}\textbf{end for}\\
\hspace{0.5cm}Add to $\mathcal{Y}$ the node for which the gain is maximal.\\
\textbf{end while}\\
\textbf{Output:} $\mathcal{Y}$.
\end{algorithmic}
\end{algorithm}

\subsection{If $\Fou_k$ is unknown}
If $\Fou_k$ is unknown, one may use a proxy that penalizes high frequencies in the solution, and thus recover an approximation of the original signal by solving the regularized problem:
\begin{align}
\label{eq:recovery_Uk_unknown}
 \sig_{\text{rec}} = \argmin_{\vec{z}\in\mathbb{R}^N} \norm{\Meas\vec{z}-\meas}^2 + \gamma\vec{z}^\adjoint\Lap^r\vec{z},
\end{align}
where $\gamma>0$ is the regularization parameter and $r$ the power of the Laplacian that controls the strength of the high frequency penalization. By differentiating with respect to $\vec{z}$, one has:
\begin{align}
  (\Meas^\adjoint\Meas + \gamma\Lap^r)\sig_{\text{rec}} = \Meas^\adjoint\meas,
\end{align}
which may be solved by direct inversion if $N$ is not too large, or --thanks to the objective's convexity-- by iterative methods such as gradient descent. 

\subsection{Reweighting in the case of random sampling}
\label{subsec:random_sampling}
In random sampling, the sample $\mathcal{A}=(\omega_1,\ldots,\omega_m)$ is a random variable, and therefore so is the measurement $\meas=\Meas\sig$. Consider two signals $\sig_1$ and $\sig_2\neq\sig_1$, and their measurements $\meas_1$ and $\meas_2$. If $\meas_1=\meas_2$, perfect reconstruction is impossible. To prevent this, one possible solution is to reweight the measurement such that the expected norm of the reweighted measurement equals the signal's norm. 
%
%
For instance, in~\cite{puy_random_2016}, $m$ nodes are drawn independently with replacement. At each draw, the probability to sample node $i$ is $p^*_i=\norm{\Fou_k^\adjoint\vec{\delta}_i}^2/k$. One has 
$\mathbb{E}\left(\norm{\ma{P}^{-1/2}\Meas\sig}^2\right)=\norm{\sig}^2$ if
\begin{align}
\label{eq:P_uncorr}
\ma{P}=\text{diag}(mp^*_{\omega_1}|\ldots|mp^*_{\omega_m}). 
\end{align}
This guarantees that if $\norm{\sig_1-\sig_2}>0$, then, for a large enough $m$, the reweighted measures will necessarily be distinct in the measurement space, \ie, $\norm{\ma{P}^{-1/2}(\meas_1-\meas_2)}>0$, ensuring that $\sig_1$ and $\sig_2$ have a chance of being recovered. 

Recovery with known $\Fou_k$ (Eq.~\eqref{eq:recovery_Uk_known}) thus transforms into 
\begin{align}
\label{eq:recovery_Uk_known_random}
 \sig_{\text{rec}} = \argmin_{\vec{z}\in\spann(\Fou_k)} \norm{\ma{P}^{-1/2}\left(\Meas\vec{z}-\meas\right)}^2;
\end{align}
and recovery with unknown $\Fou_k$ (Eq.~(\ref{eq:recovery_Uk_unknown})) transforms into 
\begin{align}
\label{eq:recovery_Uk_unknown_random}
 \sig_{\text{rec}} = \argmin_{\vec{z}\in\mathbb{R}^N} \norm{\ma{P}^{-1/2}\left(\Meas\vec{z}-\meas\right)}^2 + \gamma\vec{z}^\adjoint\Lap^r\vec{z}.
\end{align}


\section{Determinantal processes}

Denote by $[N]$ the set of all subsets of $\{1,2,\ldots,N\}$. 

\begin{definition}[Determinantal Point Process] 
\label{def:DPP} Consider a point process, \ie, a process that randomly draws an element $\mathcal{A}\in[N]$. It is determinantal if, for every $\mathcal{S}\subseteq\mathcal{A}$, 
$$\mathbb{P}(\mathcal{S}\subseteq\mathcal{A}) = \text{det}(\ma{K}_{\mathcal{S}}),$$
where $\ma{K}\in\mathbb{R}^{N\times N}$, a semi-definite positive matrix s.t. $0\preceq\ma{K}\preceq 1$, is called the marginal kernel; and $\ma{K}_\mathcal{S}$ is the restriction of $\ma{K}$ to the rows and columns indexed by the elements of $\mathcal{S}$. 
\end{definition}
\subsection{Sampling from a DPP and signal recovery}
\begin{algorithm}[tb]
   \caption{Sampling a DPP with marginal kernel $\ma{K}$~\cite{kulesza_determinantal_2012}}
   \label{alg:sampling}
\begin{algorithmic}
\STATE \textbf{Input:} Eigendecomposition of $\ma{K}: \{\mu_i, \ma{V}=(\vec{v}_1|\ldots|\vec{v}_N)\}$.\\
$\mathcal{J}\leftarrow \emptyset$\\
\textbf{for} $n=1,2,\ldots,N$ \textbf{do}:\\
\hspace{0.5cm} $\mathcal{J}\leftarrow \mathcal{J}\cup\{n\}$ with probability $\mu_n$\\
\textbf{end for}\\
$V\leftarrow\{\vec{v}_n\}_{n\in \mathcal{J}}$\\
$Y\leftarrow \emptyset$\\
\textbf{while} $|U|>0$ \textbf{do}:\\
\hspace{0.5cm} $\bm{\cdot}$ Select $i$ from $\mathcal{Y}$ with $\textrm{Pr}(i) = \frac{1}{|V|}\sum_{\vec{v}\in V}(\vec{v}^\adjoint\vec{e}_i)^2$\\
\hspace{0.5cm} $\bm{\cdot}$ $\mathcal{Y}\leftarrow \mathcal{Y}\cup\{i\}$\\
\hspace{0.5cm} $\bm{\cdot}$ $V\leftarrow V_{\perp}$, an orthonormal basis for the subspace of $V$\\\hspace{1cm} $\perp$ to $\vec{e}_i$.\\
\textbf{end while}\\
\textbf{Output:} $\mathcal{Y}$.
\end{algorithmic}
\end{algorithm}

Given the constraints on $\ma{K}$, its eigendecomposition  $\ma{K}=\sum_{i=1}^N \mu_i \vec{v}_i\vec{v}_i^\adjoint$ always exists, and Alg.~\ref{alg:sampling} provides a sample from the associated DPP~\cite{kulesza_determinantal_2012}. Note that $|\mathcal{A}|$, the number of elements  in $\mathcal{A}$, is distributed as the sum of $N$ Bernoulli trials of probability $\mu_i$ (see the ``for loop'' of Alg.~\ref{alg:sampling}). In particular:
\begin{align}
\label{eq:mean_and_var}
\hspace{-0.29cm} \mathbb{E}(|\mathcal{A}|)=\text{tr}(\ma{K})=\sum_{i=1}^N\mu_i~\text{and}~\text{Var}(|\mathcal{A}|)=\sum_{i=1}^N \mu_i (1-\mu_i).
\end{align}
Consider $\mathcal{A}$, the random variable of a DPP with kernel $\ma{K}$; and $\Meas$ its associated measurement matrix. Denote by $\pi_i=\ma{K}_{ii}$ the marginal probability that $\{i\}\subseteq\mathcal{A}$ and write:
\begin{align}
\label{eq:P_DPP}
 \ma{P}=\text{diag}(\pi_{\omega_1}|\ldots|\pi_{\omega_m})
\end{align}
\begin{proposition} $\forall\sig\in\mathbb{R}^N~~\mathbb{E}_{\mathcal{A}}\left(\norm{\ma{P}^{-1/2}\Meas\sig}^2\right)=\norm{\sig}^2.$
\end{proposition}
\begin{proof}
Let $\sig\in\mathbb{R}^N$. In the following, we need the probability $\mathbb{P}(\mathcal{A})$ that $\mathcal{A}$ is sampled (instead of marginal probabilities). One can show that $\mathbb{P}(\mathcal{A})$ is proportional to the determinant of the restriction to $\mathcal{A}$ of a matrix called $L$-ensemble~\cite{kulesza_determinantal_2012}. Note that: $\forall i~\sum_{\mathcal{A}\supset i}\mathbb{P}(\mathcal{A})=\mathbb{P}(\{i\}~\text{is sampled})=\pi_i$. 
 Thus:
\begin{align*}
 \mathbb{E}_{\mathcal{A}}&\left(\norm{\ma{P}^{-1/2}\Meas\sig}^2\right)= \mathbb{E}_{\mathcal{A}}\left(\sum_{i\in\mathcal{A}}\frac{x_i^2}{\pi_i}\right)
 =\sum_\mathcal{A}\mathbb{P}(\mathcal{A})\sum_{i\in\mathcal{A}}\frac{x_i^2}{\pi_i}\\
 &\hspace{2.3cm}=\sum_{i=1}^N x_i^2 \sum_{\mathcal{A}\supset i}\frac{\mathbb{P}(\mathcal{A})}{\pi_i}
 =\sum_{i=1}^N x_i^2 =\norm{\sig}^2\hspace{0.2cm}\qedhere
\end{align*}\end{proof}

This result is the first step to proove a restricted isometry property (RIP) for $\ma{P}^{-1/2}\Meas$ -- currently work in progress. In this paper's context, it shows how to reweight the measurement in the case of sampling from a DPP: recovery should be performed via optimization of \eqref{eq:recovery_Uk_known_random} or~\eqref{eq:recovery_Uk_unknown_random} with $\ma{P}$ as in~\eqref{eq:P_DPP}.

\subsection{If $\Fou_k$ is known: the ideal low-pass marginal kernel}
Consider the DPP defined by the following marginal kernel: 
\begin{align}
\label{def:Kk}
 \ma{K}_k = \Fou_k\Fou_k^\adjoint = \Fou\; h_{\lambda_{k}}(\ma{\Lambda})\; \Fou^\adjoint,
\end{align}
with $h_{\lambda_{k}}$ s.t. $h_{\lambda_{k}}(\lambda)=1$ if $\lambda\leq \lambda_k$ and $0$ otherwise. In GSP words, $\ma{K}_k$ is the ideal low-pass with cutting frequency $\lambda_k$. 

\begin{proposition} 
\label{thm:Kk} A sample from the DPP with marginal kernel $\ma{K}_k$ defined as in~(\ref{def:Kk}) is of size $k$ and $\forall\sig\in\spann{(\Fou_k)}$, the measurement $\meas=\ma{M}\sig + \vec{n} \in\mathbb{R}^k$ enables perfect reconstruction up to the noise level.
\end{proposition}

\begin{proof}
$\ma{K}_k = \sum_{i=1}^N\mu_i\vec{u}_i\vec{u}_i^\adjoint$, with $\mu_1=\ldots=\mu_k=1$ and $\mu_{k+1}=\ldots=\mu_N=0$. According to~\eqref{eq:mean_and_var}, 
$\mathbb{E}(|\mathcal{A}|)=k$ and $\text{Var}(|\mathcal{A}|)=0$, implying that any sample $\mathcal{A}$ has size $k$. Also, by definition of a DPP, the fact that $\mathcal{A}$ is sampled implies:
 \begin{align}
 \label{eq:detsup0}
  \text{det}({\ma{K}_{k}}_{\mathcal{A}})>0,~ \ie,~ \text{det}(\Meas\Fou_k\Fou_k^\adjoint\Meas^\adjoint)=\prod_{i=1}^k \sigma_i^2>0,
 \end{align}
implying $\sigma_1>0$. Using Prop.~\ref{thm:sigma_0} completes the proof. 
\end{proof}
 The most probable sample from this DPP is the sample $\mathcal{A}$ of size $k$ that maximizes $\text{det}(\Meas\Fou_k\Fou_k^\adjoint\Meas^\adjoint)$ \ie, $\mathcal{A}^{\text{MV}}$, the solution to the maximal volume optimisation problem~(\ref{eq:optim_MV}). 

\subsection{If $\Fou_k$ is unknown: Wilson's marginal kernel}
\label{subsec:wilson}

\begin{algorithm}[tb]
   \caption{Wilson's algorithm (see Sec.~\ref{subsec:wilson})}
   \label{alg:wilson}
\begin{algorithmic}
\STATE \textbf{Input:} Adjacency matrix $\ma{W}$ of the graph and $q>0$\\
$\mathcal{Y}\leftarrow \emptyset$, $\mathcal{V}\leftarrow \emptyset$\\
\textbf{while} $\mathcal{V}\neq\{1,\ldots,N\}$ \textbf{do}:\\
\hspace{0.5cm} $\bm{\cdot}$ Start a random walk from any node $i\in\{1,\ldots,N\}\setminus\mathcal{V}$\\\hspace{1cm} until it reaches either $\Delta$, or a node in $\mathcal{V}$.\\
\hspace{0.5cm} $\bm{\cdot}$ Erase all the loops of the trajectory and denote by $\mathcal{S}$\\ \hspace{1cm} the set of all new visited nodes (after erasure). \\
\hspace{0.5cm} $\bm{\cdot}$ $\mathcal{V}\leftarrow \mathcal{V}\cup\mathcal{S}$\\
\hspace{0.5cm} \textbf{if} the last node of the trajectory is $\Delta$ \textbf{do}: \\
\hspace{1cm} $\bm{\cdot}$ Denote by $l$ the last visited node before $\Delta$\\
\hspace{1cm} $\bm{\cdot}$ $\mathcal{Y}\leftarrow \mathcal{Y}\cup\{l\}$ \\
\textbf{Output:} $\mathcal{Y}$.
\end{algorithmic}
\end{algorithm}

First, add a node $\Delta$ to the graph, that we will call ``absorbing state'' and that we connect to all existing $N$ nodes with a weight $q$. We will propagate random walks on this extended graph. Note that the probability to jump from any node $i\neq\Delta$:
\begin{itemize}
 \item to any other node $j\neq\Delta$ is equal to $\ma{W}_{ij}/(\ma{D}_{ii}+q)$;
 \item to $\Delta$ is $q/(\ma{D}_{ii}+q)$.
\end{itemize}
Once the random walk reaches $\Delta$, it cannot escape from it. 
Wilson and Propp~\cite{wilson_generating_1996,propp1998get} introduced Algorithm~\ref{alg:wilson} in the context of spanning tree sampling. This Algorithm outputs a set $\mathcal{Y}$ of nodes by propagating loop-erased random walks.

\begin{proposition}[Corollary of Thm 3.4 in~\cite{avena_random_2013}] 
\label{thm:corollary}$\mathcal{Y}$ is a sample from a DPP with marginal kernel 
$\ma{K}_q = \Fou g_q(\Lambda) \Fou^\adjoint$, 
with $g_q(\lambda)=\frac{q}{q+\lambda}$.
\end{proposition}
\begin{proof}
Given that the weight $q$ we added is finite and constant over all nodes, Thm 3.4 and Lemma 3.3 in~\cite{avena_random_2013} state that $\mathcal{Y}$ is distributed as a DPP with marginal kernel equal to $(\Lap+q\ma{I})^{-1}q\ma{I}$ where $\ma{I}$ is the identity matrix (note that $L$ in~\cite{avena_random_2013} refers to $-\Lap$ in our paper). Decomposing $\Lap$ as $\Fou\Lambda\Fou^\adjoint$ finishes the proof. 
\end{proof}
Prop.~\ref{thm:Kk} shows that sampling with $\ma{K}_k$ is optimal in the sense it enables perfect reconstruction with only $m=k$ measurements. Unfortunately, as $N$ or $k$ increase, computing $\ma{K}_k$ becomes prohibitive. Considering $g_q$ as an approximation of $h_{\lambda_k}$, one may see sampling with $\ma{K}_q$ as an approximation of  optimal sampling with $\ma{K}_k$. Prop.~\ref{thm:corollary} has a beautiful consequence: one does not need to explicitly compute $\ma{K}_q$ (which would also be too expensive) to sample from it. In fact, Alg.~\ref{alg:wilson} efficiently samples from $\ma{K}_q$. Thus, Alg.~\ref{alg:wilson} approximates sampling from $\ma{K}_k$ without any spectral computation of $\Lap$. 

\section{Experiments}
\label{sec:experiments}
 
\noindent\textbf{The Stochastic Block Model (SBM).} 
We consider random community-structured graphs drawn from the SBM. We specifically look at graphs 
with $k$ communities of same size $N/k$. In the SBM, the probability of connection between any two nodes $i$ and $j$ is $q_1$ if they are in the same community, and $q_2$ otherwise. One can show that the average degree reads  $c=q_1\left(\frac{N}{k}-1\right)+q_2\left(N-\frac{N}{k}\right)$. 
Thus, instead of providing the probabilities $(q_1,q_2)$, one may characterize a SBM by considering $(\epsilon=\frac{q_2}{q_1},c)$. 
The larger $\epsilon$, the fuzzier the community structure. In fact, 
authors in~\cite{decelle_asymptotic_2011} show that above the critical value $\epsilon_c=(c-\sqrt{c})/(c+\sqrt{c}(k-1))$, community structure becomes  undetectable in the large $N$ limit. In the following, $N=100$, $k=2$, $c=16$ and we let $\epsilon$ vary. 

\noindent\textbf{To generate a $k$-bandlimited signal}, we draw $k$ realisations of a normal distribution with mean $0$ and variance $1$, to build the vector $\vec{\alpha}\in\mathbb{R}^k$. Then, renormalize to make sure that $\norm{\vec{\alpha}}^2=1$, and write $\sig=\Fou_k\vec{\alpha}$. In our experiments, the measurement noise is Gaussian with mean 0 and $\sigma_n=10^{-4}$.

\begin{figure*}
 \centering
a)\hspace{-0.4cm}\includegraphics[width=0.7\columnwidth]{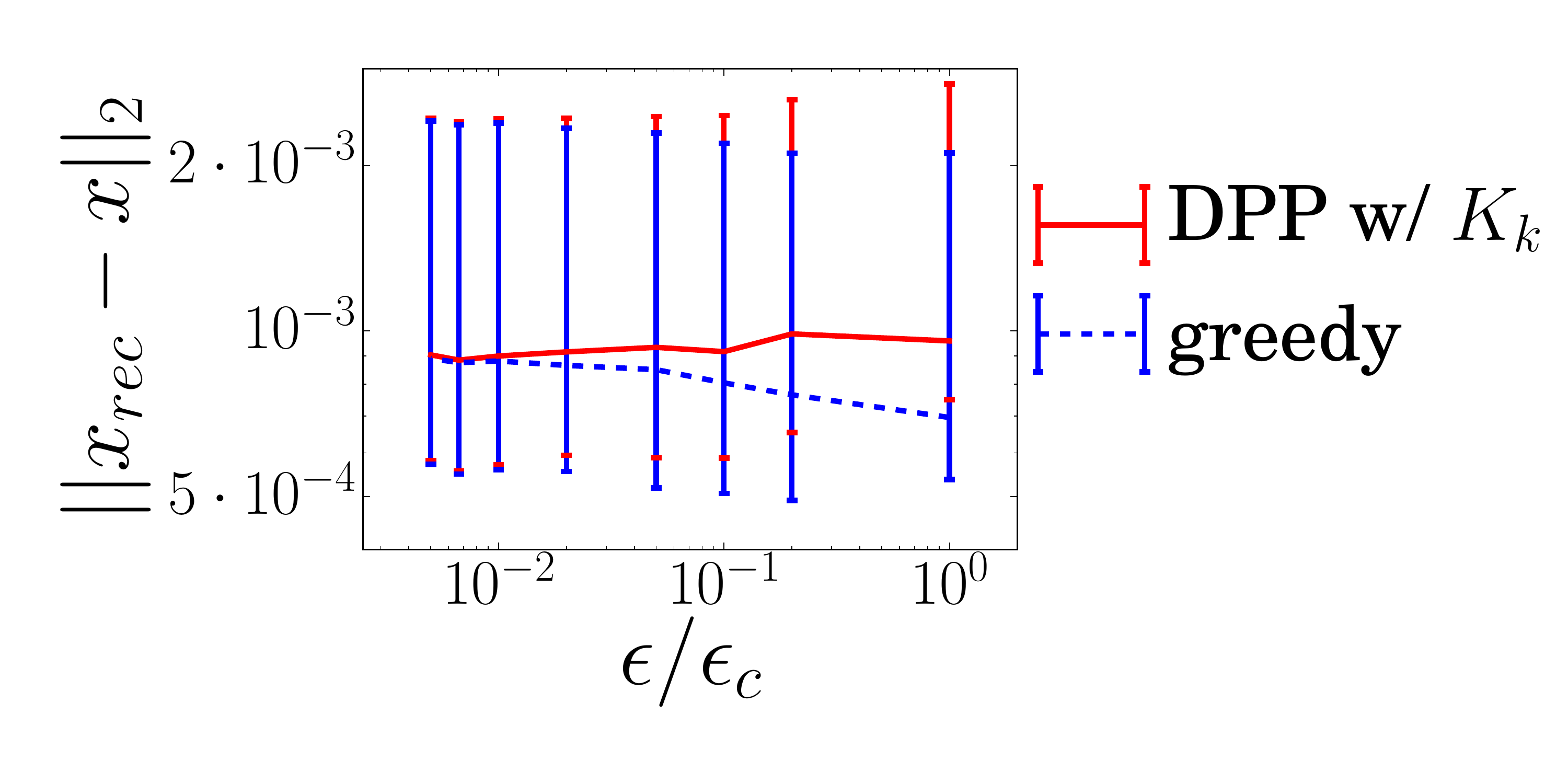} 
b)\hspace{-0.4cm}\includegraphics[width=0.7\columnwidth]{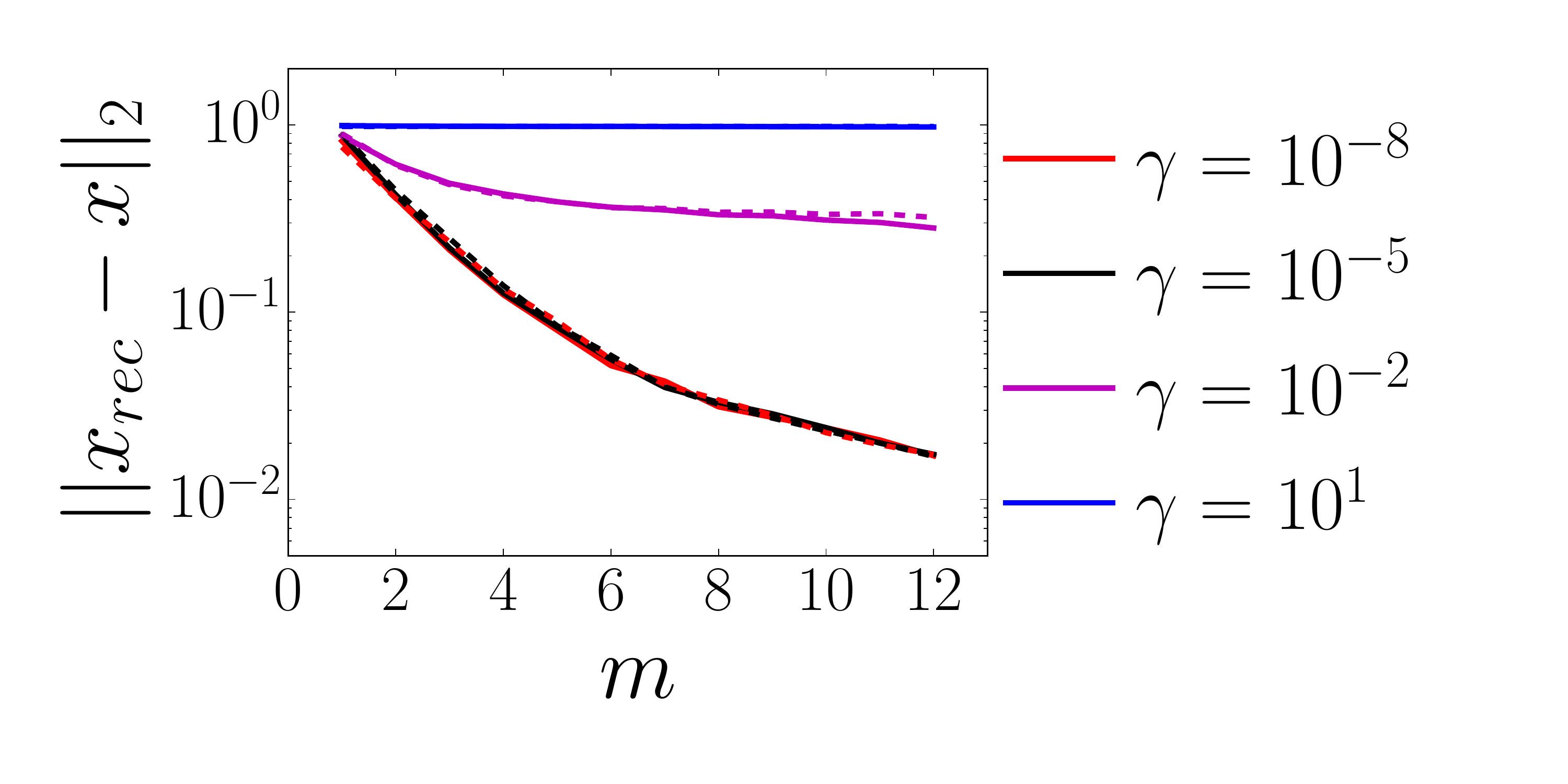}  \hspace{-0.5cm}
c)\hspace{-0.4cm}\includegraphics[width=0.7\columnwidth]{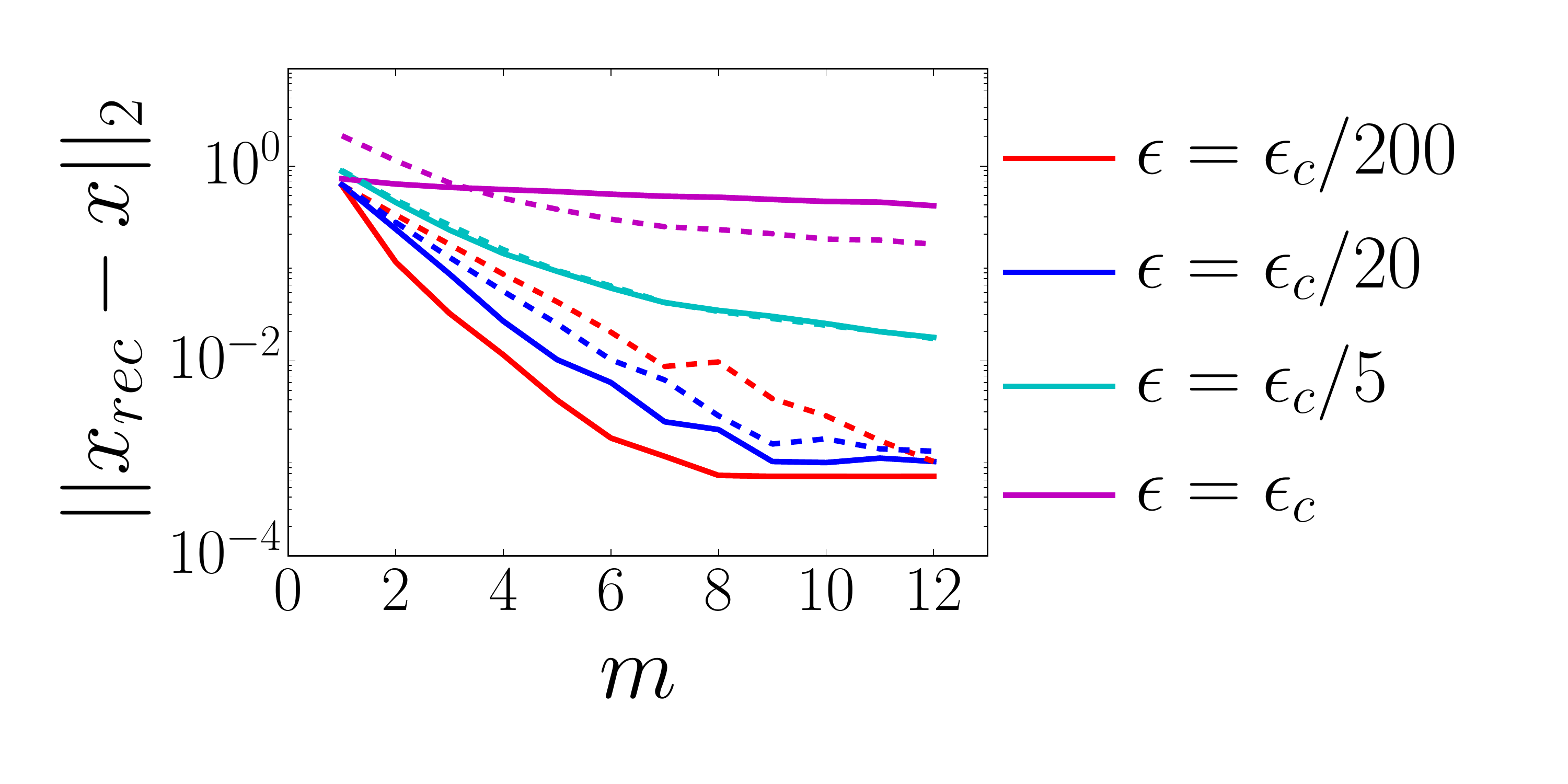} 
\caption{a) Comparison of the recovery performance after noisy measurements of $k$-bandlimited signals between DPP sampling with $\ma{K}_k$, and the 4 sampling methods presented in Sec.~\ref{subsec:greedy_options} (grouped here as ``greedy'': they perform equivalently). The intervals represent the 10- and 90-percentiles; the lines the averages, computed over $10^4$ different signals defined on $100$ different SBM realisations for each value of $\epsilon$. b)~Reconstruction performance versus the number of measurements $m$ for different values of the regularization parameter $\gamma$, for $\epsilon=\epsilon_c/5$. c)~Performance versus $m$ for different values of $\epsilon$. In b) and c), the full (resp. dashed) lines represent results obtained after DPP sampling via Wilson's algorithm (resp. after uncorrelated sampling such as in~\cite{puy_random_2016}). For each value of $m$, we average over $3500$ bandlimited signals defined on $100$ different SBM realisations. }
\label{fig:experiments}
\end{figure*}

\begin{algorithm}[tb]
   \caption{Estimation of $\pi_i=\sum_{j=1}^N g_q(\lambda_j)\vec{u}_j(i)^2$}
   \label{alg:prob_est}
\begin{algorithmic}
\STATE \textbf{Input:} Laplacian $\Lap$, $g_q:\mathbb{R}\rightarrow\mathbb{R}^+$, parameters $d$ and $n$\\
$\bm{\cdot}$ Compute $\lambda_{N}$ the largest eigenvalue of $\Lap$\\
$\bm{\cdot}$ Approximate $s=\sqrt{g_q}$ on $[0,\lambda_N]$  with a polynomial of order $d$  \ie, 
$\ma{S}=\Fou s(\Lambda)\Fou^\adjoint\simeq\Fou\sum_{l=1}^d \beta_l\Lambda^l\Fou^\adjoint=\sum_{l=1}^d \beta_l\Lap^l$\\
$\bm{\cdot}$ Generate $\ma{R}\in\mathbb{R}^{N\times n}$ with $\forall (i,j)~~\ma{R}_{ij}=\mathcal{N}(0,1/n)$\\
$\bm{\cdot}$ Compute $\ma{SR}\simeq \sum_{l=1}^d \beta_l\Lap^l\ma{R}$ recursively
\\
\textbf{Output:} $\forall i~~\hat{\pi}_i=\norm{\vec{\delta}_i^\adjoint\ma{SR}}^2$
\end{algorithmic}
\end{algorithm}

\noindent\textbf{First series: $\Fou_k$ known.} We compare DPP sampling from $\ma{K}_k$ followed by recovery using~(\ref{eq:recovery_Uk_known_random}); to deterministic sampling with  $\hat{\mathcal{A}}^{\text{WCE}}$, $\hat{\mathcal{A}}^{\text{MSE}}$, $\hat{\mathcal{A}}^{\text{MV}}$ and $\hat{\mathcal{A}}^{\text{maxvol}}$ presented in Sec.~\ref{subsec:greedy_options} followed by recovery using~(\ref{eq:recovery_Uk_known}). 
We show in Fig.~\ref{fig:experiments}a) the recovery performance with respect to $\epsilon/\epsilon_c$: all greedy methods perform similarly, and on average slightly outperform the DPP-based method, especially as the graphs become less structured. 

\noindent\textbf{Second series: $\Fou_k$ unknown.}  We compare negatively correlated sampling from $\ma{K}_q$ (using Wilson's algorithm) followed by recovery using~(\ref{eq:recovery_Uk_unknown_random}) with $\ma{P}$ as in~\eqref{eq:P_DPP}; to the uncorrelated random sampling~\cite{puy_random_2016} where nodes are sampled independently with replacement from $p^*$ with  $p^*_i=\norm{\Fou_k^\adjoint\vec{\delta}_i}^2/k$, followed  by recovery using~(\ref{eq:recovery_Uk_unknown_random}) with $\ma{P}$ as in~\eqref{eq:P_uncorr}. Several  comments are in order. i)~the number of samples from Wilson's algorithm is not known in advance, but we know its expected value $\mathbb{E}(|\mathcal{A}|)=\sum_i q/(q+\lambda_i)$. In order to explore the behavior around the critical number of measurements $m=k$, one needs to choose $q$ s.t. $\sum_i q/(q+\lambda_i)\simeq k$, but without computing the $\lambda_i$! To do so, we follow the proposition of~\cite{avena_random_2013} (see discussion around Eq.~(4.14)) based on a few runs of Alg.~\ref{alg:wilson}. Once a sample of size approximately equal to $k$ is exhibited, and for fair comparison, one uses the same number of nodes for the uncorrelated sampling. ii)~exact computation of $p^*$ for the uncorrelated sampling requires $\Fou_k$. We follow the efficient Alg.~1 of~\cite{puy_random_2016} to approximate $p^*$. iii)~recovery using $\ma{P}$ as in~\eqref{eq:P_DPP} requires to know $\pi_i=\sum_{j=1}^N g_q(\lambda_j)\vec{u}_j(i)^2$ for each node $i$. To estimate $\pi_i$ without spectral decomposition of $\Lap$, an elegant solution is via fast graph  filtering of random signals, as in Alg.~\ref{alg:prob_est}. Building upon the Johnson-Lindenstrauss lemma, one can show  that $\hat{\pi}_i$ concentrates around $\pi_i$ for $n=O(\log{N})$ (see a similar proof in~\cite{tremblay_accelerated_2016}). We fix $n=20\log{N}$ and $d=30$.

The reconstruction's parameters are $\gamma$ and $r$. Following~\cite{puy_random_2016}, $r$ is fixed to $4$. Fig.~\ref{fig:experiments}b) shows that the reconstruction performance saturates at $\gamma=10^{-5}$. With these values of $r$ and $\gamma$, we compare DPP sampling with $\ma{K}_q$ vs uncorrelated sampling in Fig.~\ref{fig:experiments}c) with respect to the number of measurements, for different values of $\epsilon$. The lower $\epsilon$, \ie~the stronger the community structure, the better DPP sampling is compared to uncorrelated sampling. 
In terms of computation time, a precise comparison of both methods versus the different parameters is out of this paper's scope.  Nevertheless, to give an idea of the method's scalability, for a SBM with $N=10^5$ (resp. $10^6$), $k=2$, $c=16$ and $\epsilon=\epsilon_c/5$, and with $q=5\cdot10^{-4}$, Wilson's algorithm outputs in average 5 (resp. 36) samples in a mean time of 7 (resp. 90) seconds, using Python on a laptop.

\section{Conclusion}
We first show that sampling from a DPP with marginal kernel $\ma{K}_k=\Fou_k\Fou_k^\adjoint$ \ie~the projector onto the first $k$ graph Fourier modes, outputs a subset of size $k$ that enables perfect reconstruction. The robustness to noise of the reconstruction is comparable to state-of-the-art greedy sampling algorithms. Moreover, in the important and fairly common case where the first $k$ eigenvectors of the Laplacian are not computable, we show that Wilson's algorithm (Alg.~\ref{alg:wilson}) may be leveraged as an efficient graph sampling scheme that approximates sampling from $\ma{K}_k$. Preliminary experiments on the SBM suggest that in the cases where  the number of samples $m$ needs to stay close to $k$ and where the community structure of the graph is strong, sampling with Wilson's algorithm enables a better reconstruction than the state-of-the-art uncorrelated random sampling. Further investigation, both theoretical and experimental, is necessary to better specify  which families of graphs are favorable to DPP sampling, and which are not.

\bibliographystyle{IEEEtran}
\bibliography{EUSIPCO_Library.bib}

\end{document}

%% file: macro+variables.tex
\newcommand{\Graph}{\ensuremath{\set{G}}}

\newcommand{\Lap}{\ensuremath{\ma{L}}}
\newcommand{\Fou}{\ensuremath{\ma{U}}}

\newcommand{\Meas}{\ensuremath{\ma{M}}}

\newcommand{\sig}{\ensuremath{\vec{x}}}
\newcommand{\meas}{\ensuremath{\vec{y}}}

\newcommand{\nbVert}{\ensuremath{N}}

\newcommand{\nbClass}{\ensuremath{k}}

\newcommand{\Rbb}{\ensuremath{\mathbb{R}}} 
\newcommand{\Nbb}{\ensuremath{\mathbb{N}}}

\renewcommand{\leq}{\ensuremath{\leqslant}}
\renewcommand{\geq}{\ensuremath{\geqslant}}

\newcommand{\adjoint}{\ensuremath{{\intercal}}}

\newcommand{\norm}[1]{\ensuremath{\left\| #1\right\|}}

\newcommand{\ma}[1]{\ensuremath{\mathsf{#1}}}
\renewcommand{\vec}[1]{\ensuremath{\bm{#1}}}

\newcommand{\set}[1]{\ensuremath{\mathcal{#1}}}

\newcommand{\spann}{\ensuremath{{\rm span}}}

\newcommand{\ie}{\textit{i.e.}}

\DeclareMathOperator*{\argmin}{argmin}

\newtheorem{definition}[theorem]{Definition}

%% file: paper.bbl
\begin{thebibliography}{10}
\providecommand{\url}[1]{#1}
\csname url@samestyle\endcsname
\providecommand{\newblock}{\relax}
\providecommand{\bibinfo}[2]{#2}
\providecommand{\BIBentrySTDinterwordspacing}{\spaceskip=0pt\relax}
\providecommand{\BIBentryALTinterwordstretchfactor}{4}
\providecommand{\BIBentryALTinterwordspacing}{\spaceskip=\fontdimen2\font plus
\BIBentryALTinterwordstretchfactor\fontdimen3\font minus
  \fontdimen4\font\relax}
\providecommand{\BIBforeignlanguage}[2]{{%
\expandafter\ifx\csname l@#1\endcsname\relax
\typeout{** WARNING: IEEEtran.bst: No hyphenation pattern has been}%
\typeout{** loaded for the language `#1'. Using the pattern for}%
\typeout{** the default language instead.}%
\else
\language=\csname l@#1\endcsname
\fi
#2}}
\providecommand{\BIBdecl}{\relax}
\BIBdecl

\bibitem{shuman_emerging_2013}
D.~Shuman, S.~Narang, P.~Frossard, A.~Ortega, and P.~Vandergheynst, ``The
  emerging field of signal processing on graphs: {Extending} high-dimensional
  data analysis to networks and other irregular domains,'' \emph{Signal
  Processing Magazine, IEEE}, vol.~30, no.~3, pp. 83--98, May 2013.

\bibitem{pesenson_sampling_2008}
I.~Pesenson, ``Sampling in {Paley}-{Wiener} spaces on combinatorial graphs,''
  \emph{Transactions of the AMS}, vol. 360, no.~10, 2008.

\bibitem{anis_efficient_2016}
A.~Anis, A.~Gadde, and A.~Ortega, ``Efficient {Sampling} {Set} {Selection} for
  {Bandlimited} {Graph} {Signals} {Using} {Graph} {Spectral} {Proxies},''
  \emph{IEEE Transactions on Signal Processing}, vol.~64, no.~14, pp.
  3775--3789, 2016.

\bibitem{chen_discrete_2015}
S.~Chen, R.~Varma, A.~Sandryhaila, and J.~Kovačević, ``Discrete signal
  processing on graphs: Sampling theory,'' \emph{IEEE Transactions on Signal
  Processing}, vol.~63, no.~24, pp. 6510--6523, 2015.

\bibitem{puy_random_2016}
G.~Puy, N.~Tremblay, R.~Gribonval, and P.~Vandergheynst, ``Random sampling of
  bandlimited signals on graphs,'' \emph{Applied and Computational Harmonic
  Analysis}, 2016, in press.

\bibitem{gama_rethinking_2016-1}
F.~Gama, A.~G. Marques, G.~Mateos, and A.~Ribeiro, ``Rethinking sketching as
  sampling: {Linear} transforms of graph signals,'' in \emph{50th {Asilomar}
  {Conf}. {Signals}, {Systems} and {Computers}}, 2016.

\bibitem{shuman_multiscale_2016}
D.~I. Shuman, M.~J. Faraji, and P.~Vandergheynst, ``A {Multiscale} {Pyramid}
  {Transform} for {Graph} {Signals},'' \emph{IEEE Transactions on Signal
  Processing}, vol.~64, no.~8, pp. 2119--2134, Apr. 2016.

\bibitem{narang_compact_2013}
S.~Narang and A.~Ortega, ``Compact {Support} {Biorthogonal} {Wavelet}
  {Filterbanks} for {Arbitrary} {Undirected} {Graphs},'' \emph{Signal
  Processing, IEEE Transactions on}, vol.~61, no.~19, pp. 4673--4685, Oct.
  2013.

\bibitem{drineas_randnla:_2016}
P.~Drineas and M.~W. Mahoney, ``\BIBforeignlanguage{en}{{RandNLA}: randomized
  numerical linear algebra},'' \emph{\BIBforeignlanguage{en}{Communications of
  the ACM}}, vol.~59, no.~6, 2016.

\bibitem{tremblay_compressive_2016}
N.~Tremblay, G.~Puy, R.~Gribonval, and P.~Vandergheynst, ``Compressive spectral
  clustering,'' in \emph{Proceedings of the 33 rd {International} {Conference}
  on {Machine} {Learning} ({ICML})}, vol.~48, 2016, pp. 1002--1011.

\bibitem{avena_random_2013}
L.~Avena and A.~Gaudillière, ``On some random forests with determinantal
  roots,'' \emph{arXiv preprint arXiv:1310.1723}, 2013.

\bibitem{deshpande_efficient_2010}
A.~Deshpande and L.~Rademacher, ``Efficient volume sampling for row/column
  subset selection,'' in \emph{51st {Annual} {IEEE} {Symposium} on the
  Foundations of {Computer} {Science} {(FOCS)}}, 2010, pp. 329--338.

\bibitem{goreinov_maximal-volume_2001}
S.~A. Goreinov and E.~E. Tyrtyshnikov, ``The maximal-volume concept in
  approximation by low-rank matrices,'' \emph{Contemporary Mathematics}, vol.
  280, pp. 47--52, 2001.

\bibitem{kulesza_determinantal_2012}
A.~Kulesza and B.~Taskar, ``Determinantal {Point} {Processes} for {Machine}
  {Learning},'' \emph{Found. and Trends in Mach. Learn.}, vol.~5, no. 2–3,
  pp. 123--286, 2012.

\bibitem{goreinov_how_2008}
S.~A. Goreinov, I.~V. Oseledets, D.~V. Savostyanov, E.~E. Tyrtyshnikov, and
  N.~L. Zamarashkin, ``How to find a good submatrix,'' \emph{Research Report
  08-10, ICM HKBU, Kowloon Tong, Hong Kong}, pp. 08--10, 2008.

\bibitem{wilson_generating_1996}
D.~B. Wilson, ``Generating random spanning trees more quickly than the cover
  time,'' in \emph{Proceedings of the twenty-eighth annual {ACM} symposium on
  {Theory} of computing}.\hskip 1em plus 0.5em minus 0.4em\relax ACM, 1996, pp.
  296--303.

\bibitem{propp1998get}
J.~G. Propp and D.~B. Wilson, ``How to get a perfectly random sample from a
  generic markov chain and generate a random spanning tree of a directed
  graph,'' \emph{Journal of Algorithms}, vol.~27, no.~2, 1998.

\bibitem{decelle_asymptotic_2011}
A.~Decelle, F.~Krzakala, C.~Moore, and L.~Zdeborová, ``Asymptotic analysis of
  the stochastic block model for modular networks and its algorithmic
  applications,'' \emph{Phys. Rev. E}, vol.~84, no.~6, 2011.

\bibitem{tremblay_accelerated_2016}
N.~Tremblay, G.~Puy, P.~Borgnat, R.~Gribonval, and P.~Vandergheynst,
  ``Accelerated {Spectral} {Clustering} {Using} {Graph} {Filtering} {Of}
  {Random} {Signals},'' in \emph{{International} {Conference} on {Acoustics},
  {Speech} and {Signal} {Processing} ({ICASSP})}, 2016, pp. 4094--4098.

\end{thebibliography}
